\documentclass{article}
\usepackage{float}
\usepackage{main}
\usepackage{graphicx}
\usepackage[utf8]{inputenc} % allow utf-8 input
\usepackage[T1]{fontenc}    % use 8-bit T1 fonts
\usepackage{hyperref}       % hyperlinks
\usepackage{url}            % simple URL typesetting
\usepackage{booktabs}       % professional-quality tables
\usepackage{amsfonts}       % blackboard math symbols
\usepackage{nicefrac}       % compact symbols for 1/2, etc.
\usepackage{microtype}      % microtypography
\usepackage{lipsum}

\usepackage{xurl}
\usepackage{microtype}
\usepackage{graphicx}
\usepackage{booktabs} % for professional tables

\usepackage{multirow}
\usepackage{amsmath,amsfonts,amssymb}
\usepackage{amsthm}
\usepackage{mathtools}
\usepackage{tikz}
\usepackage{cite}
\usepackage{mathrsfs}
\usepackage{subfig}
\usepackage{caption}
\usepackage{comment}

\usepackage{array}
\usepackage[round]{natbib}
\usepackage{color}
\hypersetup{
    citecolor=blue,
    colorlinks,
    linkcolor=black,
    urlcolor=black,
    linktoc=all
}
\usepackage{hyperref}

\usepackage[ruled,vlined]{algorithm2e}
\SetKwInput{KwInput}{Input}
\SetKwInput{KwOutput}{Output}

\usepackage{lipsum}
\usepackage{todo}

\DeclareMathOperator*{\argmin}{argmin}
\DeclareMathOperator*{\vecteur}{vec}

\DeclareMathOperator*{\KPSVD}{KPSVD}

\newtheorem{theorem}{Theorem}[section]

\newtheorem{prop}{Proposition}[section]

\usepackage{makecell}
\newcommand\independent{\protect\mathpalette{\protect\independenT}{\perp}}
\def\independenT#1#2{\mathrel{\rlap{$#1#2$}\mkern2mu{#1#2}}}

\renewcommand{\cite}[1]{\citep{#1}}
\usepackage{hyperref}

\newcommand{\indexdata}{t}
\newcommand{\rearrange}{\mathcal{Z}}
\newcommand{\Aat}{\mathtt{A}}
\newcommand{\Ggt}{\mathtt{G}}

\usepackage{authblk}
\title{Efficient Approximations of the Fisher Matrix in Neural Networks using Kronecker Product Singular Value Decomposition}

\author[1,2,*]{Abdoulaye Koroko}
\author[1]{Ani Anciaux-Sedrakian}
\author[1]{Ibtihel Ben Gharbia}
\author[3]{Valérie Garès}
\author[3]{Mounir Haddou}
\author[1]{Quang Huy Tran}
 
\affil[1]{IFP Energies nouvelles, 1 et 4 avenue de Bois Préau, 92852 Rueil-Malmaison Cedex, France}
\affil[2]{CentraleSupélec – Université Paris-Saclay, 3 Rue Joliot Curie, 91190 Gif-sur-Yvette, France}
\affil[3]{Univ Rennes, INSA, CNRS, IRMAR - UMR 6625, F-35000 Rennes, France}
\affil[*]{Corresponding author: \texttt{abdoulaye.koroko@ifpen.fr}}

\begin{document}
\maketitle

\begin{abstract}
Several studies have shown the ability of natural gradient descent to minimize the objective function more efficiently than ordinary gradient descent based methods.  However, the bottleneck of this approach for training deep neural networks lies in the prohibitive cost of solving a large dense linear system corresponding to the Fisher Information Matrix (FIM) at each iteration. This has motivated various approximations of either the exact FIM or the empirical one. The most sophisticated of these is KFAC, 
%[Martens \& Grosse, Optimizing neural networks with Kronecker-factored approximate
%curvature , PMLR 37: 2408--2417, 2015],
which involves a Kronecker-factored block diagonal approximation of the FIM. With only a slight additional cost, a few improvements of KFAC from the standpoint of accuracy are proposed. The common feature of the four novel methods is that they rely on a direct minimization problem, the solution of which can be computed via the Kronecker product singular value decomposition technique. Experimental results on the three standard deep auto-encoder benchmarks showed that they provide more accurate approximations to the FIM. Furthermore, they outperform KFAC and state-of-the-art first-order methods in terms of optimization speed.
\end{abstract}

\section{Introduction}

In Deep Learning, the Stochastic Gradient Descent (SGD) method \citep{RobbinsMonro1951} and its variants are currently the prevailing methods for training neural networks. To solve the problem
$$\argmin_{\theta \in \mathbb{R}^p} h(\theta) := \frac{1}{n}\sum_{\indexdata=1}^n L(y_{\indexdata},f_{\theta}(x_{\indexdata})) ,$$ 
where $h$ denotes the empirical risk associated with the training data $\mathcal{T} = \{ (x_1,y_1),(x_2,y_2),\hdots (x_n,y_n) \}$ and the loss function $L$, the batch SGD method produces iterates
$$\theta_{k+1} = \theta_{k}-\alpha_{k}\nabla_\theta h({\mathcal{S}_k,\theta_{k}}),$$ 
where $\alpha_k>0$ stands for the learning rate and where
$$\nabla_\theta h(\mathcal{S}_k,\theta_{k})= \frac{1}{|\mathcal{S}_k|}
\sum_{ (x_{\indexdata},y_{\indexdata}) \in \mathcal{S}_k} \nabla_\theta L(y_{\indexdata},f_{\theta_k}(x_{\indexdata}))$$
is a batch approximation of the full gradient $\nabla_\theta h(\theta_k) = \frac{1}{n}\sum_{\indexdata=1}^n \nabla_\theta L(y_{\indexdata}, f_{\theta_k}(x_{\indexdata}))$
on a random subset $\mathcal{S}_{k} \subset \mathcal{T}$.

Despite its ease of implementation and great popularity in the machine learning community, the SGD method, like all other first-order methods, is known to have limited effectiveness (requires many iterations in order to converge or even simply diverges) for non-convex objective functions, 
%in the case where the \huy{objective} function is non-convex, 
as is the case in 
%the objective functions of
deep neural networks.
 
In classical optimization, second-order methods are known for their efficiency in terms of convergence speed compared to first-order methods. A second-order iteration reads
$$ \theta_{k+1} = \theta_k - \alpha_k [H(\theta_k)]^{-1}\nabla_\theta h(\theta_k),$$ 
where $H(\theta_k) \in \mathbb{R}^{p\times p}$ is the curvature matrix of $h$ at $\theta_k$. The matrix $H$ can be the Hessian matrix $\nabla^2_{\theta\theta} h$ as in the Newton-Raphson method,
%; this is not suitable in the non-convex configuration since in this situation the Hessian matrix is not necessarily positive semi-definite and therefore 
the drawback of which is that $-H^{-1}\nabla_\theta h$ is not guaranteed to be a descent direction \cite{NocedalWright06}. It is wiser to replace the Hessian matrix by a surrogate such as the Generalized Gauss-Newton matrix \citep{Schraudolph2002} or the Fisher Information Matrix (FIM) \citep{Amari1998}, which are always positive semi-definite. Unfortunately, second-order methods remain impractical for deep neural networks where the number of parameters can quickly become very large (tens of millions), making it impossible to compute and to store, let alone to invert $H$.
 
A first way to avoid assembling and storing the matrix $H$ is the inexact resolution of the linear system by Conjugate Gradient (CG), which requires only matrix-vector products. This \emph{Hessian-free} philosophy \citep{martens2010} is still expensive, since the CG must be run with a significant number of iterations before reaching an acceptable convergence. 
 
An alternative is to consider the direct inversion of a diagonal approximation to $H$, as in \citep{becker} for the Hessian matrix or in \citep{adagrad,rmsprop,KingmaBa2015} for the empirical FIM. The reader is referred to \citep{Martens2014,KBH2019} for the difference between the empirical and the exact FIM (both are estimators of the true FIM but the second one uses sampled outputs from the model distribution). Another approach is to use a low-rank approximation of the Hessian matrix such as BFGS \citep{broyden1970,fletcher1970,goldfard1970,shanno1970} or its low-memory version L-BFGS \citep{nocedal1989}, which is better suited to deep learning. Nevertheless, the trouble with diagonal and low-rank approximations is that they are very rough and therefore give rise to less efficient algorithms than a well-tuned SGD.
 
More advanced methods resort to a block-diagonal approximation of a curvature matrix. \citet{LRMB2008} and \citet{olivier2015} use respectively a block-diagonal approximation of the empirical and exact FIM where each block contains  the weights associated to a particular neuron. Based on early ideas in \citep{Heskes2000,PascanuBengio2013,PZK2014}, a new family of methods under the name of KFAC have recently emerged \cite{MartensGrosse2015,GrosseMartens2016,BaGrosseMartens2017,MBJ2018,GLBBV2018}. Thanks to a Kronecker-factored layer-wise block-diagonal approximation of the FIM, the KFAC methods have proven to be more powerful than a well-tuned SGD. Following similar lines of thought, \citet{BRB2017} and \citet{GRB2020} also proposed efficient approximations of respectively GGN and Hessian matrices for training Multi-Layer Perceptrons (MLP). 
 
%In this work, we pay particular attention to the KFAC method. We have investigated this method and found that 
The fundamental assumption on which KFAC hinges is the
%for the Kronecker product factorization to the blocks of the FIM is the strong assumption of 
independence between activations and pre-activation derivatives. We believe that this premise, the theoretical foundation of which is unclear, is at the root of a poor quality of the FIM approximation. This is why, in this work, we wish to put forward four Kronecker-factored block-diagonal approximations that aim at more accurately representing the FIM by removing this assumption. To this end, we minimize the Frobenius norm of the difference between the original matrix and a prescribed form for the approximation, which is achievable through the Kronecker product singular value decomposition. Tests carried out on the three standard deep auto-encoder benchmarks showed that our proposed methods outperform KFAC both in terms of FIM approximation quality and optimization speed of the objective function.

The  paper  is  organized  as  follows:  Section 2  introduces the natural gradient and KFAC methods. Section 3 proposes the above mentioned novel approximations.  In  Section 4, we present and comment several numerical experiments. Finally, the conclusion overviews the work undertaken in this research and outlines directions for future study.

%\label{submission}
\medskip
\section{Background and notation}
The notations used in this paper are fairly similar to those introduced in \cite{MartensGrosse2015}. We consider an $\ell$-layer feedforward neural network $f_{\theta}$ parametrized by
$$\theta = [\text{vec}(W_1)^T,\text{vec}(W_2)^T, \hdots, \text{vec}(W_{\ell})^T]^T \in\mathbb{R}^p ,$$
where $W_i \in\mathbb{R}^{d_i\times (d_{i-1}+1)}$ is the weights matrix associated to layer $i$ and ``vec'' is the operator that vectorizes a matrix by stacking its columns together. This network transforms an input $x = : a_0\in\mathbb{R}^{d_0}$ to an output $z=f_{\theta}(x)$ by the sequence
$$
s_i = W_i \Bar{a}_{i-1}, \quad a_i=\sigma_i(s_i), \quad \text{for}\; i \; \text{from} \; 1\; \text{to} \; \ell,
$$ 
terminated by $z := a_\ell \in \mathbb{R}^{d_\ell}$. Here, $\Bar{a}_{i-1} = (1, a_i^T)^T$ is the augmented activation vector (value 1 is used for the bias) and $\sigma_i$ the activation function at layer $i$. The number of neurons at layer $i$ is $d_i$ and the total number of parameters is $p= \sum_{i=1}^{\ell} d_i (d_{i-1} + 1)$.

For a given input-target pair $(x,y)$, the gradient of the loss $L(y,f_{\theta}(x))$ w.r.t to the weights is computed by the back-propagation algorithm \cite{LeCun1988}. For convenience, we adopt the shorthand notation $\mathcal{D}v = \nabla_v L$ for the derivative of $L$ w.r.t any variable $v$, as well as the special symbol $g_i = \mathcal{D}s_i$ for the preactivation derivative. Starting from $\mathcal{D}a_{\ell} = \partial_z L(y,z=a_\ell)$, we perform
$$
 g_i= \mathcal{D}a_i \odot \sigma'_i(s_i), \quad
 \mathcal{D}W_{i}= g_{i}\Bar{a}_{i-1}^T, \quad
 \mathcal{D}a_{i-1} = W_i^T g_i,
$$ 
for $i$ from $\ell$ to 1, where $\odot$ denotes the component-wise product. Finally, the gradient $\nabla_\theta L$ is retrieved as
$$
\mathcal{D}\theta = [\text{vec}(\mathcal{D}W_1)^T,\text{vec}(\mathcal{D}W_2)^T, \ldots, \text{vec}(\mathcal{D}W_{\ell})^T]^T .
$$
%Training the network consists in seeking for the best parameter $\theta$ that minimizes the empirical risk: $$h\left(\theta\right) = \frac{1}{n}\sum_{i=1}^n L\left(y_i,f_{\theta}(x_i)\right)$$.

\subsection{Natural Gradient Descent}
The loss function $L(y,z)$ 
%\huy{measuring the discrepancy} between the network prediction $z$ and the actual label $y$ 
is now assumed to take the form 
\[
L(y,z) = -\log(p(y|z))=-\log(p(y|f_{\theta}(x)))=-\log(p(y|x,\theta)),
\]
where $p(y|x,\theta)$ is the density function of the probability distribution $P_{y|x}(\theta)$ governing the output $y$ around the value $f_{\theta}(x)$ predicted by the network. Note that $P_{y|x}(\theta)$ is multivariate normal for the standard square loss function, multinomial for the cross-entropy one. %and in the case of any loss function $L(y,z)$, it is possible to choose $P_{y|x}(\theta)$ such that $P_{y|x}(\theta) \propto \exp\left({-L(y,z)}\right)$.
Then, the natural gradient descent method \cite{Amari1998} is defined as
$$\theta_{k+1} = \theta_k - \alpha_k \left[F(\theta_k)\right]^{-1}\nabla_{\theta} h(\theta_k),$$
where
$$
 F(\theta)= \mathbb{E}_{x\sim Q_x, y\sim P_{y|x}(\theta)}[\mathcal{D}\theta (\mathcal{D}\theta)^T]
$$ 
is the FIM associated to the network parameter $\theta$. The expectation is taken according to the distribution $Q_x$ of the input data $x$ and the conditional distribution $P_{y|x}(\theta)$ of the the network 's output prediction $y$. For brevity and without any risk of ambiguity, we will omit the subscripts for the expectation and write $\mathbb{E}$ instead of $\mathbb{E}_{x\sim Q_x, y\sim P_{y|x}}(\theta)$.

%While the ordinary gradient descent gives the steepest descent \huy{direction} in the  space of parameters,
The natural gradient method can be seen as the steepest descent method in the space of model's probability distributions with the metric induced by the Kullback-Leibler (KL) divergence \cite{AmariNagaoka2000}. Indeed, it can be shown that for some constant scaling factor $\lambda >0$,
$$
-\frac{1}{\lambda} [F(\theta)]^{-1}\nabla_\theta h(\theta) = \!\! \argmin_{d: \text{KL} [P_{y|x}(\theta) \, \| \, P_{y|x}(\theta+d) ]\,=\, c} \!\! h(\theta+d).
$$
The appealing property of the natural gradient $F^{-1}\nabla h$ is that it has an intrinsic geometric interpretation, regardless of the actual choice of parameters. A thorougher discussion can be found in \cite{Martens2014}. 

It follows from the definition of the FIM that
$$
F =  \mathbb{E} [\mathcal{D}\theta (\mathcal{D}\theta)^T ]
 =   \begin{bmatrix} 
    F_{1,1}  & \dots & F_{1,\ell} \\
    \vdots  & & \vdots \\
     F_{\ell,1}  & \dots & F_{\ell,\ell}
    \end{bmatrix}, 
$$
in which the block
$$ 
F_{i,j} =   \mathbb{E}[\text{vec}(\mathcal{D}W_i)\text{vec}(\mathcal{D}W_j)^T]
        =   \mathbb{E}[\Bar{a}_{i-1}\Bar{a}_{j-1}^T\otimes g_ig_j^T]
$$
is a $d_i(d_{i-1}+1) \times d_j(d_{j-1}+1)$ matrix. We recall that the Kronecker product $ A\otimes B$ between two matrices $A \in \mathbb{R}^{m_A\times n_A}$ and $B\in \mathbb{R}^{m_B\times n_B}$ is the $m_A m_B\times n_A n_B$ matrix
$$
A\otimes B =  \begin{bmatrix} 
    A_{1,1}B  & \dots & A_{1,n_A}B \\
    \vdots &  & \vdots \\
   \, A_{m_A,1}B & \dots   & A_{m_A,n_A}B \;
    \end{bmatrix} .
$$
The blocks of $F$ can be given the following meaning: $F_{i,i}$  contains second-order statistics of weight derivatives on layer $i$, while $F_{i,j},i\neq j$ represents correlation between weight derivatives of layers $i$ and $j$.

\subsection{KFAC method}\label{sub:kfac}
The Kronecker-factored approximate curvature (KFAC) method introduced by \cite{MartensGrosse2015} is grounded on two assumptions that provide a computationally efficient approximation of $F$. 

The first assumption is that $F_{i,j}=0$ for $i\neq j$. In other words, weight derivatives in two different layers are uncorrelated. This results in block-diagonal approximation
$$
F\approx \text{diag}(F_{1,1},F_{2,2},\hdots F_{\ell,\ell}) .
$$
%Since $\forall{i}, F_{i,i}\in \mathbb{R}^{(d_{i-1}+1)d_i\times (d_{i-1}+1)d_i}$ where $d_i$ and $d_{i-1}$ are the number of neurons in layers $i$ and $i-1$ respectively,
This first approximation is insufficient, insofar as the blocks of $F_{i,i}$ are very large for neural networks with high number of units in layers. A further approximation is in order.

The second assumption is that of independent activations and derivatives (IAD): activations  and pre-activation derivatives are independent.
i.e $\forall{i}$, $ a_{i-1} \independent g_i$. 
This allows each block $F_{i,i}$ to be factorized into a Kronecker product of two smaller matrices, i.e.,
 \begin{equation} \label{eq:1}
 \begin{aligned}
     F_{i,i} & = \mathbb{E}[\Bar{a}_{i-1}\Bar{a}_{i-1}^T\otimes g_ig_i^T]\\
             & \approx \mathbb{E}[\Bar{a}_{i-1}\Bar{a}_{i-1}^T]\otimes\mathbb{E}[ g_i g_i^T] \\ 
    & =: \Bar{A}_{i-1}^{\text{KFAC}} \otimes {G}_{i}^{\text{KFAC}},
    \end{aligned}
    \end{equation}
with $\Bar{A}_{i-1}^{\text{KFAC}} =\mathbb{E}[\Bar{a}_{i-1}\Bar{a}_{i-1}^T]  \in \mathbb{R}^{(d_{i-1}+1)\times (d_{i-1}+1) }$ and $G_{i}^{\text{KFAC}} =\mathbb{E}[ g_ig_i^T]  \in \mathbb{R}^{d_{i}\times d_{i} }$.
    
These two assumptions yield the KFAC approximation
$$
F_{\text{KFAC}} = \text{diag}(\Bar{A}_{0}^{\text{KFAC}}\otimes {G}_{1}^{\text{KFAC}},  \hdots,  \Bar{A}_{\ell-1}^{\text{KFAC}}\otimes {G}_{\ell}^{\text{KFAC}} ).
$$
KFAC has been extended to convolution neural networks (CNN) by \citet{GrosseMartens2016}. However, due to weight sharing in convolutional layers, it was necessary to add two extra assumptions regarding spatial homogeneity and spatially uncorrelated derivatives.

The decisive advantage of $F_{\text{KFAC}}$ is that it can be inverted in a very economical way. Indeed, owing to the properties 
$(A\otimes B)^{-1} = A^{-1} \otimes B^{-1}$ and $(A\otimes B)\text{vec}(X)= \text{vec}(BXA^T)$ of the Kronecker product, 
the approximate natural gradient $F_{\text{KFAC}}^{-1}\nabla h$ can be evaluated as
\begin{equation}\label{eq:precond}
    F_{\text{KFAC}}^{-1}\nabla h = \begin{bmatrix}
    \text{vec} (G_{1}^{-1}(\nabla_{W_1}h)\Bar{A}_{0}^{-1}) \\ \vdots \\ 
    \text{vec} (G_{\ell}^{-1}(\nabla_{W_{\ell}}h)\Bar{A}_{\ell-1}^{-1})
    \end{bmatrix} ,
\end{equation}
where the KFAC superscripts are dropped from now on to alleviate notations.
This drastically reduces computations and memory requirements, since we only need to store, invert and multiply the smaller matrices $\Bar{A}_{i-1}$'s and $G_i$'s.

In practice, because the curvature changes relatively slowly \cite{MartensGrosse2015}, the factors $(\Bar{A}_{i-1}, G_i)$ are computed at every $T_1$ iterations and their inverses at every $T_2$ iterations. Moreover, $(\Bar{A}_{i-1},G_i)$ are estimated using exponentially decaying moving average. At iteration $k$, let $(\Bar{A}_{i-1}^{\text{old}},G_i^{\text{old}})$ be the factors previously computed at iteration $k-T_1$ and $(\Bar{A}_{i-1}^{\text{new}}, G_i^{\text{new}})$ be those computed with the current mini-batch. Then, setting $\rho = \min(1-1/k,\alpha)$ with $\alpha \in [0,1]$, we have
\begin{alignat*}{2}
\Bar{A}_{i-1} &= \rho\Bar{A}_{i-1}^{\text{old}} & & + (1-\rho)\Bar{A}_{i-1}^{\text{new}},\\
 G_i &= \rho G_i^{\text{old}} & & +(1-\rho) G_i^{\text{new}}. 
\end{alignat*}

Another crucial ingredient of KFAC is the Tikhonov regularization to enforce invertibility of $F_{\text{KFAC}}$. The straightforward damping $F_{\text{KFAC}} + \lambda I$ deprives us of the possibility of applying the formula $(A\otimes B)^{-1}=A^{-1}\otimes B^{-1}$.
%or damping (adding $\lambda I$ to $F_{\text{KFAC}}$). 
%Since $F_{\text{KFAC}}$ is block-diagonal, this is equivalent to adding $\lambda I$ (with the corresponding shape) to each of its blocks. However this makes the use of the
%formula $(A\otimes B)^{-1}=A^{-1}\otimes B^{-1}$ impossible.
To overcome this issue, \citet{MartensGrosse2015} advocated the more judicious Kronecker product regularization
$$
\widetilde{F}_{i,i} = (\Bar{A}_{i-1}+\pi_i \lambda^{1/2} I) \otimes
(G_i+ \pi_i^{-1} \lambda^{1/2} I) 
$$
%to approximate $\Bar{A}_{i-1}\otimes G_i + \lambda I$ by  
%$\Bar{A}_{i-1}^{\lambda}\otimes G_i^{\lambda}$ with $\Bar{A}_{i-1}^{\lambda}=\Bar{A}_{i-1}+\pi_i\lambda^{1/2}$ and $G_i^{\lambda}=G_i+\frac{1}{\pi_i}\lambda^{1/2}$
where $\lambda > 0$ and
$$
\pi_i = \sqrt{\frac{\text{tr}(\Bar{A}_{i-1})/(d_{i-1}+1)}{\text{tr}(G_i)/d_i}} .
$$

\section{Four novel methods}
%We keep the first approximation made in KFAC which approaches $F$ by the block-diagonal matrix $\Tilde{F}$. As for the second approximation (IAD assumption), we believe that it hurts the quality of KFAC approximation and propose a series of novel factorizations to the blocks of $\Tilde{F}$ that remove this assumption.
While staying within the framework of the first assumption (block-diagonal approximation), we now design four new methods that break free from the second hypothesis (IAD) in order to achieve a better accuracy: KPSVD, Deflation, Lanczos-bidiagonalization and KFAC-corrected.

\subsection{KPSVD}\label{sec:kpsvd}
In our first method, called KPSVD, the factors $(\Bar{A}_{i-1},G_{i})$ are specified as the arguments of the best possible approximation of $F_{i,i}$ by a single Kronecker product. Thus,
\begin{align}
      (\Bar{A}_{i-1}, {G}_{i})  & = \argmin_{(R,S)} \|F_{i,i}-R\otimes S\|_{F}\label{eq:2}\\
                            & = \argmin_{(R,S)} \|\mathbb{E}[\Bar{a}_{i-1}\Bar{a}_{i-1}^T\otimes g_ig_i^T]-R\otimes S\|_F , \nonumber
\end{align}
where $\|\cdot\|_F$ denotes Frobenius norm. Although the minimization problem \eqref{eq:2} has already been introduced in abstract linear algebra by van Loan \cite{VanLoan2000,VanLoanPitsianis1993}, it has never been considered in the context of neural networks, at least to the best of our knowledge. Anyhow,  it can be solved at a low cost by means of the Kronecker product singular value decomposition technique \cite{VanLoan2000}. To write down the solution, we need the following notion. Let
$$
M =\begin{bmatrix} 
    M_{1,1} & \dots & M_{1,d} \\
     M_{2,1} & \dots & M_{2,d}\\
    \vdots  & & \vdots \\
     M_{d,1}  & \dots & M_{d,d} 
    \end{bmatrix} \in \mathbb{R}^{d'd\times d'd}
$$
be a uniform block matrix, that is, $M_{\mu,\nu} \in \mathbb{R}^{d'\times d'}$ for all $(\mu,\nu)\in \{1,\ldots d\}^2$. The {\em zigzag rearrangement} operator $\rearrange$ converts $M$ into the matrix
$$
\rearrange(M)=\begin{bmatrix}
    \text{vec}(M_{1,1})^T\\
    \vdots\\
    \text{vec}(M_{d,1})^T\\
    \vdots\\
    \text{vec}(M_{1,d})^T\\
    \vdots\\
    \, \text{vec}(M_{d,d})^T\,\\
    \end{bmatrix} \in \mathbb{R}^{d^2 \times (d')^2},
$$
by flattening out each block in a column-wise order and by transposing the resulting vector. This operator is to be applied to each $M=F_{i,i}$ with $d= d_{i-1}+1$ and $d'= d_i$.

\begin{theorem}\label{th:th1}
Any solution of \eqref{eq:2} is also a solution of the ordinary rank-1 matrix approximation problem
\begin{equation}\label{eq:ordrankone}
 (\vecteur(\Bar{A}_{i-1}^{\KPSVD}), \vecteur({G}_{i}^{\KPSVD}) )  = 
\argmin_{(R,S)} \|\rearrange(F_{i,i})-\vecteur(R)\vecteur(S)^T\|_{F} .
\end{equation}
\end{theorem}
\begin{proof}
    See appendix \ref{appendix:th1}.
\end{proof}

Problem \eqref{eq:ordrankone} is solved as follows. Let $U^T\rearrange(F_{i,i})V=\Sigma$ be the singular value decomposition (SVD) of $\rearrange(F_{i,i})$. Let $\sigma_1$ be the greatest singular value of $\rearrange(F_{i,i})$ and $(u_1, v_1)$ be the associated left and right singular vectors. A solution to \eqref{eq:ordrankone} is, 
$$
\Bar{A}_{i-1}^{\text{KPSVD}} = \sqrt{\sigma_1} \, \mathrm{MAT}(u_1), 
\quad
{G}_{i}^{\text{KPSVD}} = \sqrt{\sigma_1} \, \mathrm{MAT}(v_1),
$$
where ``MAT,'' the converse of ``vec,'' turns a vector into a matrix.
The question to be addressed now is how to compute $u_1$, $v_1$ and $\sigma_1$. We recommend the power SVD algorithm (see appendix \ref{appendix:svd}), which only requires the matrix-vector multiplications $\rearrange(F_{i,i})v$ and $\rearrange(F_{i,i})^T u$. These operations can be performed without explicitly forming $F_{i,i}$ or $\rearrange(F_{i,i})$, as elaborated on in the upcoming Proposition.

\begin{prop}\label{prop:1}
For all $u\in \mathbb{R}^{(d_{i-1}+1)^2}$ and $v \in \mathbb{R}^{d_{i}^2}$, 
\begin{align*}
\rearrange(F_{i,i})v & =  \mathbb{E}[\,g_i^T V g_i \,\mathrm{vec}(\Bar{a}_{i-1}\Bar{a}_{i-1}^T) \,]  ,\\
\rearrange(F_{i,i})^Tu & = \mathbb{E}[\,\Bar{a}_{i-1}^T U \Bar{a}_{i-1}\, \mathrm{vec}(g_ig_i^T) \,]  ,
\end{align*}
with $U = \mathrm{MAT}(u)$ and $V = \mathrm{MAT}(v)$.
\end{prop}
\begin{proof}
    See appendix \ref{appendix:prop1}.
\end{proof}

\subsection*{Estimating $\rearrange(F_{i,i})v$ and $\rearrange(F_{i,i})^Tu$}
Let us consider a batch $\mathcal{B}=\left\{ (x_1,y_1),\hdots, (x_m,y_m) \right\}$ drawn from the training data $\mathcal{T}$. We recall that the expectation is taken with respect to both $Q_x$ (data distribution over inputs $x$) and $P_{y|x}(\theta)$ (predictive distribution of the network). To estimate $\rearrange\left(F_{i,i}\right)v$ and $\rearrange\left(F_{i,i}\right)^Tu$, we use the Monte-Carlo method as suggested by \citet{MartensGrosse2015}: we first compute the statistics $\Bar{a}_{i-1}$'s and $g_i$'s during an additional back-propagation performed using targets $y$'s sampled from  $P_{y|x}(\theta)$ and then set
\begin{align*}
\rearrange(F_{i,i})v & \approx \frac{1}{m} \sum_{\indexdata=1}^m  g_{i,\indexdata}^T V g_{i,\indexdata} \, \text{vec}(\Bar{a}_{i-1,\indexdata} \Bar{a}{}_{i-1,\indexdata}^T)  , \\
\rearrange(F_{i,i})^Tu  &\approx \frac{1}{m} \sum_{\indexdata=1}^m  \Bar{a}_{i-1,\indexdata}^T U \Bar{a}_{i-1,\indexdata} \, \text{vec}( g_{i,\indexdata} g_{i,\indexdata}^T),
\end{align*}
where the subscript t loops over the data points in the batch B.

So far, we have not paid attention to the symmetry of the matrices $(\Bar{A}_{i-1}^{\text{KPSVD}}, {G}_{i}^{\text{KPSVD}})$ in problem \eqref{eq:2}. It turns out that symmetry is automatic, while positive semi-definiteness occurs for some solutions to be selected.
\begin{prop}\label{prop:2}
All solutions $(\Bar{A}_{i-1}^{\KPSVD}, {G}_{i}^{\KPSVD})$ of problem \eqref{eq:2} are symmetric. Besides, we can select solutions for which these matrices are positive semi-definite.
\end{prop}

\begin{proof}
    See appendix \ref{appendix:prop2}.
\end{proof}

\subsection{Kronecker rank-2 approximation to $F_{i,i}$}
%Consider the same $A$ and $\rearrange(A)$ as in the theorem \ref{th:th1}. Let $\tilde{r}$ be the rank of $\rearrange(A)$ and $\rearrange(A)=U\Sigma V^T$ the SVD of $\rearrange(A)$ with $U=\left[u_1|u_2|\hdots|u_{MN}|\right]$, $V=\left[v_1|v_2|\hdots|v_{pq}|\right]$ and $\Sigma = \text{diag}\left(\sigma_1,\hdots,\sigma_{\tilde{r}},0,\hdots 0\right)$ with $\sigma_1\geq \sigma_2 \geq \hdots \geq \sigma_{\tilde{r}}$.
%From \cite{VanLoan2000}, $A = \sum_{k=1}^{\tilde{r}}\sigma_k U_k\otimes V_k$ with $U_k \in \mathbb{R}^{M\times N}$ and $V_k \in \mathbb{R}^{p\times q}$ are such that $u_k=\text{vec}\left(U_k\right)$ and $v_k=\text{vec}\left(V_k\right)$.

Since the KPSVD method of \S\ref{sec:kpsvd} is merely a Kronecker rank-1 approximation of $F_{i,i}$, it is most natural to look for higher order approximations. The two methods presented in this section are based on seeking a Kronecker rank-2 approximation $R\otimes S + P\otimes Q$ of $F_{i,i}$ that achieves
\begin{equation}\label{eq:rank2pb}
\min_{(R,S,P,Q)} \|F_{i,i} - (R\otimes S + P\otimes Q) \|_F .
\end{equation}
Again, the zigzag rearrangement operator $\rearrange$ enables us to reformulate \eqref{eq:rank2pb} as an ordinary rank-2 matrix approximation problem. To determine a solution of the latter, there are two techniques in practice: \emph{deflation} \cite{Saad} and \emph{Lanczos bi-diagonalization} \cite{Golub}.

\subsubsection{Deflation}
The rank-1 factors $(R,S)$ and the rank-2 factors $(P,Q)$ are computed successively, one after another:
\begin{enumerate}
    \item Apply the power SVD algorithm  to $\rearrange(F_{i,i})$  to compute  $(R,S)$ so as to minimize $\|F_{i,i} -  R\otimes S\|_F$. The solution is known to be $(R,S)= (\Bar{A}_{i-1}^{\text{KPSVD}}, G_i^{\text{KPSVD}})$.
    \item Let $\widehat{F}_{i,i} = F_{i,i} - R\otimes S$. Apply the power SVD algorithm to $\rearrange(\widehat{F}_{i,i})$ to compute $(P, Q)$ so as to minimize $\| \widehat{F}_{i,i} - P\otimes Q\|_F$.
    \item Set $F_{i,i}\approx R\otimes S + P\otimes Q$.
\end{enumerate}
In step 2, we need to calculate the matrix-vector products $\rearrange(\widehat{F}_{i,i})v$ and $\rearrange(\widehat{F}_{i,i} )^Tu$. These operations can be done efficiently without explicitly forming $\widehat{F}_{i,i}$ or $\rearrange(\widehat{F}_{i,i})$. Indeed,
\begin{align*}
\rearrange(\widehat{F}_{i,i})v & =  \rearrange(F_{i,i})v-\rearrange(R\otimes S)v ,\\
\rearrange(\widehat{F}_{i,i})^Tu & = \rearrange(F_{i,i})^Tu-\rearrange(R\otimes S)^Tu.
\end{align*}
On one hand, we know how compute $\rearrange(F_{i,i})v$ and $\rearrange(F_{i,i})^Tu$ from Proposition \ref{prop:1}. On the other hand, it is not difficult to show that
\begin{align*}
\rearrange(R\otimes S)v & = \left<\text{vec}(S),v\right>\text{vec}(R),\\
\rearrange(R\otimes S)^Tu & = \left<\text{vec}(R),u\right>\text{vec}(S), 
\end{align*}
where $\left<\cdot,\cdot\right>$ stands for the dot product.

\subsubsection{Lanczos bi-diagonalization}
In contrast to deflation, the Lanczos bi-diagonalization algorithm (see appendix \ref{appendix:lanczos}) computes $(R,S)$ and $(P,Q)$ at the same time. It does so by simultaneously computing the two largest singular values $\sigma_1\geq \sigma_2$ of $\rearrange(F_{i,i})$ with the associated singular vectors $(u_1,v_1)$ and $(u_2,v_2)$. Once these singular elements are determined, it remains to set
\begin{alignat*}{2}
R &= \sqrt{\sigma_1} \,\text{MAT}(u_1), &\quad S &=  \sqrt{\sigma_1}\, \text{MAT}(v_1),\\
P &= \sqrt{\sigma_2} \,\text{MAT}(u_2), &\quad Q &=  \sqrt{\sigma_2}\,\text{MAT}(v_2).
\end{alignat*}
Similarly to KPSVD, we only have to perform the matrix-vector multiplications  $\rearrange(F_{i,i})v$ and $\rearrange(F_{i,i})^Tu$ without forming and storing $F_{i,i}$ or $\rearrange(F_{i,i})$.

In pratice, it is advisable to implement the restarted version of the algorithm \cite{Saad}, which consists of three steps:
\begin{enumerate}
        \item \textbf{Start}: Choose an initial vector $q^{(0)}$ and a dimension $K$ for the Krylov subspace.
        \item \textbf{Iterate}: Perform Lanczos bidiagonalization algorithm (appendix \ref{appendix:lanczos}).
        \item \textbf{Restart}: Compute the desired singular vectors. If stopping criterion satisfied, stop. Else set $q^{(0)} = \text{linear combination of singular vectors}$ and go to 2. 
    \end{enumerate}

\subsection{KFAC-CORRECTED}
Another idea is to simply add an {\em ad hoc} correction to the KFAC approximation. Put another way, we consider
$$
F_{i,i} \approx \Bar{A}_{i-1}^{\text{KFAC}} \otimes {G}_{i}^{\text{KFAC}}
+ \Bar{A}_{i-1}^{\text{corr.}}\otimes {G}_{i}^{\text{corr.}},
$$
using the best possible correctors, that is,
\begin{equation}\label{eq:kfaccorrpb}
(\Bar{A}_{i-1}^{\text{corr.}}, {G}_{i}^{\text{corr.}}) = \\
\argmin_{(P,Q)} \| F_{i,i} - \Bar{A}_{i-1}^{\text{KFAC}} \otimes {G}_{i}^{\text{KFAC}}  - P\otimes Q\|_F .
\end{equation}
%We can write $F_{i,i} = \Bar{A}_{i-1}\otimes {G}_{i}+ \mathcal{E}_i$ where $\mathcal{E}_i = F_{i,i}-\Bar{A}_{i-1}\otimes {G}_{i}$ is the approximation error. 
%
%We propose to write $\mathcal{E}_i$ as $\mathcal{E}_i = \Bar{A}_{i-1}^{\text{corrected}}\otimes {G}_{i}^{\text{corrected}}$ using the SVD algorithm and then set: 
Again, the solution of \eqref{eq:kfaccorrpb} can be computed by applying the power SVD algorithm to the matrix $\rearrange(F_{i,i} - \Bar{A}_{i-1}^{\text{KFAC}} \otimes {G}_{i}^{\text{KFAC}} )$. The matrix-vector multiplications required can be done in the same way as in the \emph{deflation} method without explicitly forming and storing the matrices.

\section{Inversion of $A\otimes B+ C\otimes D$}\label{sec:inversion}
For each of the last three methods, we need to solve a linear system of the form $(A\otimes B+ C\otimes D)u=v$ in an efficient way. This is far from obvious, since due to the sum, the well-known and powerful identities $(A\otimes B)^{-1} = A^{-1}\otimes B^{-1}$ and  $(A\otimes B)^{-1}\text{vec}(X)=\text{vec}(B^{-1}XA^{-T})$ can no longer be applied.

There are many good methods to compute $u$, but the most appropriate for our problem is that of \citet{MartensGrosse2015}, since it takes advantage of symmetry and definiteness of the matrices. Below is a summary of the algorithm, the full details of which are in \cite{MartensGrosse2015}.
%It requires computations of SVD, of matrix square roots and matrix-matrix multiplications. The technique can be summarised by the following steps (See \cite{MartensGrosse2015} for the full description).
\begin{enumerate}
    \item Compute $A^{-1/2}$, $B^{-1/2}$ and the symmetric eigen/SVD-decompositions 
    \begin{align*}
    A^{-1/2}CA^{-1/2} &= E_1S_1E_1^T,\\
    B^{-1/2}DB^{-1/2} &= E_2S_2E_2^T,
    \end{align*}
    where $S_{1,2}$ are diagonal and $E_{1,2}$ are orthogonal.
    \item Set $K_1=A^{-1/2}E_1$, $K_2 = B^{-1/2}E_2$. Then,
    $$
    u=\text{vec}(K_2[(K_2^TVK_1)\oslash(\mathbf{1}\mathbf{1}^T+s_2s_1^T)]K_1^T),
    $$ 
    where $E\oslash F$ denotes the Hadamard or element-wise division of $E$ by $F$, $s_{1,2}=\text{diag}(S_{1,2})$, $\mathbf{1}$ vector of ones and $V = \text{MAT}(v)$. Note that $K_{1,2}$, $s_{1,2}$ can be stored and reused for different choices of $v$.
\end{enumerate}
Despite the numerous steps involved, the inversion of $A\otimes B+ C\otimes D$ is much cheaper than that of $F_{i,i}$. Indeed, if $d$ denotes the number of neurons the current layer, then the matrices $A,B,C,D$ are of size $d\times d$ each, and therefore the inversion of $A\otimes B+C\otimes D$ has $O(d^2)$ memory requirement and $O(d^3)$ computational cost. Meanwhile, since $F_{i,i}$ is a matrix size $d^2\times d^2$, its inversion requires $O(d^4)$ memory and $O(d^6)$ flops.

\section{Experiments}
We have evaluated our proposed methods as well as KFAC, SGD and ADAM on the three standard deep-auto-encoder problems used for benchmarking neural network optimization methods \cite{martens2010,sluskeretal2013,MartensGrosse2015,BRB2017}. The benchmarks consist of training three different auto-encoder architectures with CURVES, MNIST and FACE datasets respectively. See appendix \ref{appendix:data} for a complete description of the network architectures and datasets. In our experiments, all our proposed methods as well as KFAC use approximations of the exact FIM $F$. 
Experiments were performed with PyTorch framework \cite{pytorch} on supercomputer with Nividia Ampere A100 GPU and AMD Milan@2.45GHz CPU.

The precision value $\epsilon$ for power SVD and Lanczos bi-diagonalization algorithm was set to $10^{-6}$. Also for these two algorithms, we used a warm-start technique which means that the final results of the previous iteration are used as a starting point (instead of a random point) for the current iteration. This has resulted in a faster convergence.
In all experiments, the batch sizes used are $256$, $512$ and $1024$ for CURVES, MNIST and FACES datasets respectively.

We first evaluate the approximation qualities of the FIM and then report the results on performance of the optimization objective.

\subsection{Approximation qualities of the FIM}\label{sub:FIM}
We investigated how well our proposed methods and KFAC approximate blocks of the exact FIM. To do so, we computed for each of the problems the exact FIM and its different approximations of the $5$th layer of the network. For a fair comparison, the exact FIM as well as its different approximations were computed during the same optimization process with an independent optimizer (SGD or ADAM). We ran two independent tests with SGD and ADAM optimizers respectively and ended up with the same results. We therefore decided  to report only the results obtained with ADAM.
Let $F$ be the exact FIM of the $5$th layer of the network and $\hat{F}$ be any approximation to $F$ ($\hat{F}$ is in the form $A\otimes B$ for KFAC and KPSVD, and $R\otimes S+P\otimes Q$ for KFAC corrected, Deflation and Lanczos). We measured the following two types of error: 
\begin{itemize}
    \item \textbf{Error 1}: Frobenius norm error between $F$ and $\hat{F}$ : $\|F-\hat{F}\|_F / \|F\|_F$;
    \item \textbf{Error 2}: $\ell_2$ norm error between the spectra of $F$ and $\hat{F}$:   $\|\text{spec}(F)-\text{spec}(\hat{F})\|_2 / \|\text{spec}(F)\|_2$ where $\text{spec}(M)$ denotes the spectrum of $M$ and $\|\cdot\|_2$ is the $\ell_2$ norm.
\end{itemize}

Note that here the Fisher matrices were estimated without the exponentially decaying averaging
scheme which means that only the mini-batch at iteration $k$ is used to compute the Fisher matrices at this iteration.

\begin{figure*}[thbp]
\centering

\begin{tabular}{ccc}
      &\includegraphics[width=0.5\textwidth]{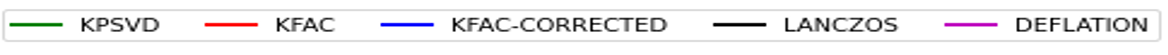} & \\
      \subfloat[CURVES]{\includegraphics[width=45mm]{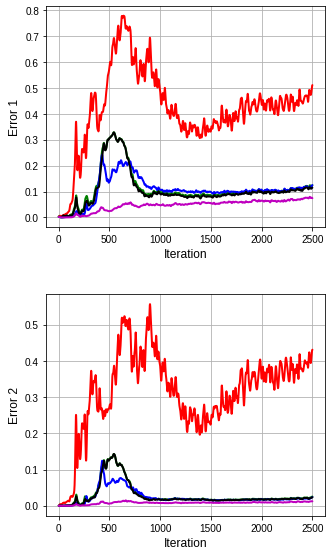}} & \hspace{-3.0cm}
      \subfloat[MNIST]{\includegraphics[width=45mm]{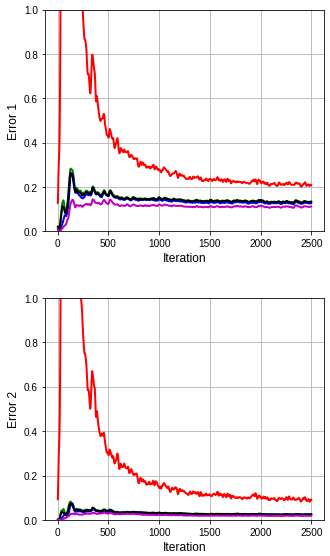}}& \hspace{-3.0cm}
      \subfloat[FACES]{\includegraphics[width=45mm]{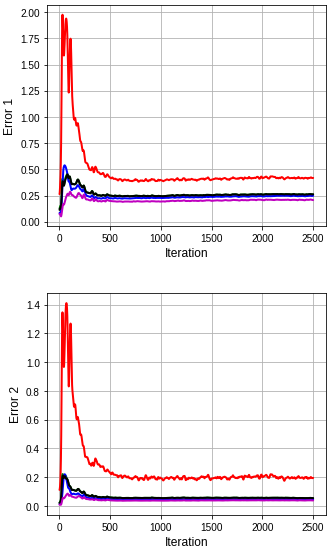}}
\end{tabular}

\caption{Comparison between FIM approximation qualities of our methods and KFAC. For each problem, at each training iteration of the network with ADAM optimizer, the exact FIM and its different approximations are computed for layer $5$ of the network. \textbf{Error 1} and \textbf{Error 2} described in subsection \ref{sub:FIM} are measured. For the sake of visual comparison between different methods, the display scale in the axis of ordinates was deliberately restricted to $[0,1]$ for the MNIST problem. It thus seems that the error curves for KFAC, whose peak amplitudes are about $6.5$, are truncated.}
\label{fig:fishers}
\end{figure*}

As we can see in Figure \ref{fig:fishers}, for each of the problems, the Deflation method gives the best approximation, followed by the other methods.  Although Deflation and Lanczos bi-diagonalization may appear as two implementations of the same idea (i.e., computing the two largest singular vectors), the former turns out to be more robust than the latter, in the sense that it converges much faster to the two dominant pairs and also produces a much smaller error. This accounts for the difference in performance between the two methods. The \textbf{Error 1}  and \textbf{Error 2} made by our different methods remain lower than those caused by KFAC throughout the optimization process. This suggests that our methods give a better approximation to the Fisher than KFAC, and that increasing the rank does improve the quality of approximation. One can go further in this direction if there is no prohibitive extra cost.

\subsection{Optimization performance}
We now consider the network optimization in each of the three problems. We have evaluated our methods against KFAC and the baselines (SGD and ADAM). Here the different approximations to the FIM were computed using the exponentially decaying technique as described in \S\ref{sub:kfac}. The decay factor $\alpha$ was set to $0.95$ as in \cite{MartensGrosse2015}.
Since the goal of KFAC as well as our methods is optimization performance rather than generalization, we performed Grid Search for each method and selected hyperparameters that gave a better reduction to the training loss. The learning rate $\eta$  and the damping parameter $\lambda$ are in range $\{10^{-1}, 10^{-2}, 10^{-3}, 10^{-4}, 3\cdot 10^{-1}, 3\cdot 10^{-2}, 3\cdot 10^{-3},  3\cdot 10^{-4}\}$,  and the clipping parameter $c$ belongs to $\{10^{-2}, 10^{-3}\}$ (see appendix \ref{appendix:clipping} for definition of $c$). Note that damping and clipping are used only in KFAC and our proposed methods. Update frequencies $T_1$ and $T_2$ were set to $100$. The momentum parameters were $\beta=0.9$ for SGD and $(\beta_1,\beta_2) = (0.9,0.999)$ for ADAM. 

\begin{figure*}[thbp]
\begin{tabular}{c}
      \includegraphics[width=1.0\textwidth]{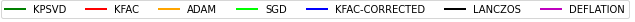} \\
      \subfloat[CURVES]{\includegraphics[width=180mm]{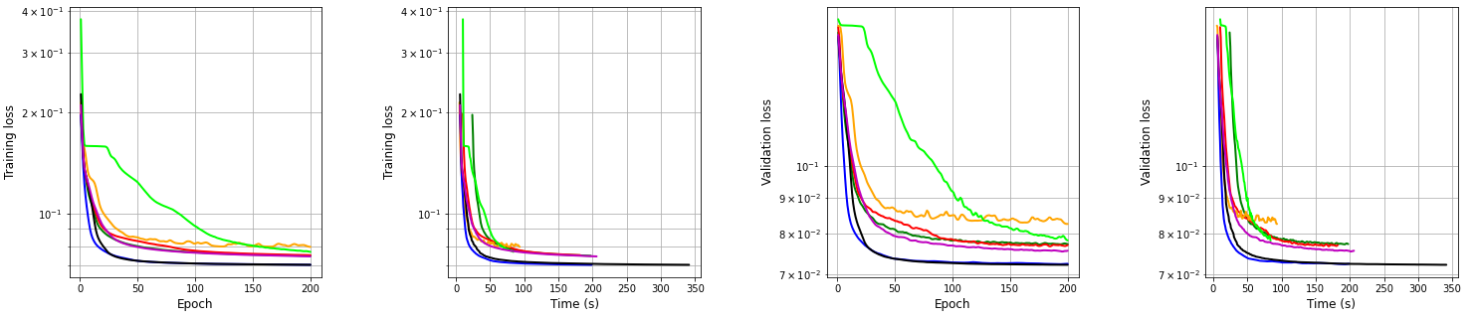}} \\
      \subfloat[MNIST]{\includegraphics[width=180mm]{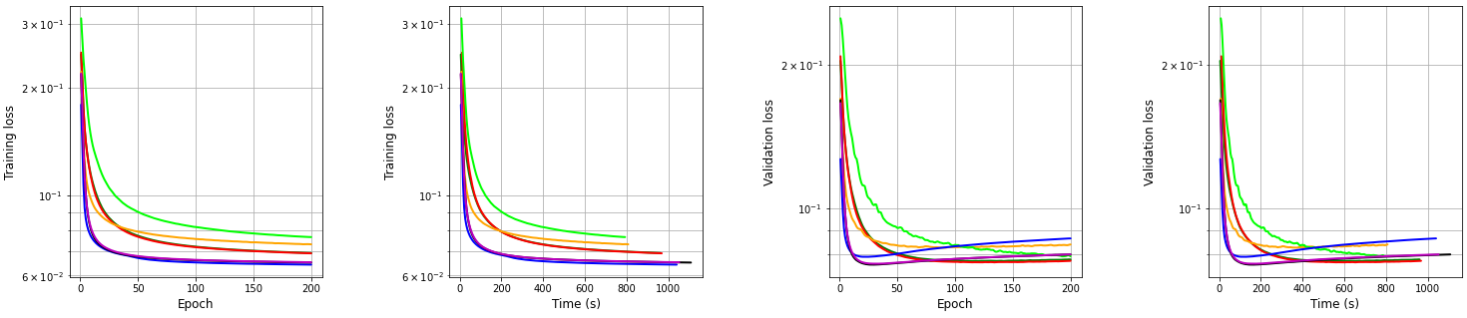}}\\
      \subfloat[FACES]{\includegraphics[width=180mm]{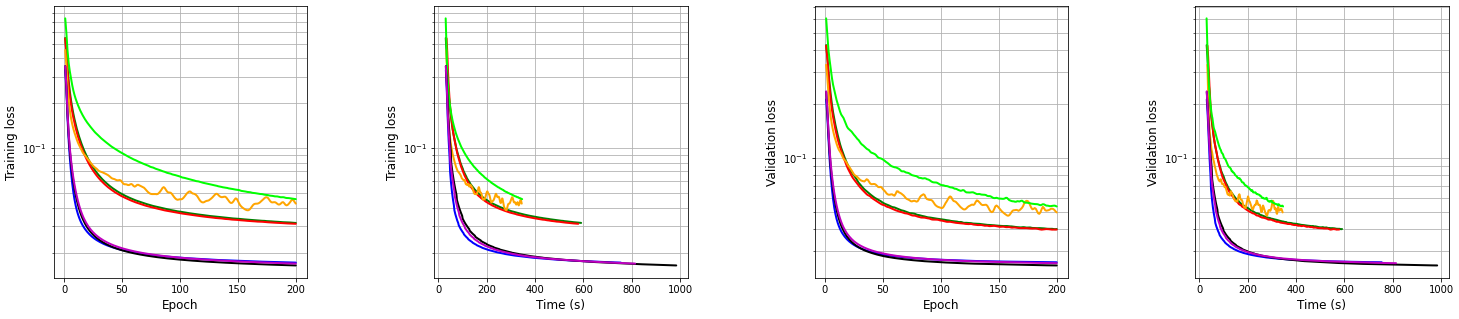}}\\
      
\end{tabular}
\caption{Comparison of optimization performance of different algorithms on each of the 3 problems (CURVES \textbf{top} row, MNIST \textbf{middle} row and FACES \textbf{bottom} row). For each problem, from left to right, \textbf{first} figure displays training loss \emph{vs} epoch, \textbf{second} one represents training loss \emph{vs} time, the \textbf{third} depicts validation loss \emph{vs} epoch and the \textbf{last} displays validation loss \emph{vs} time.}
\label{fig:train}
\end{figure*}

Figure \ref{fig:train} shows the performance of the different optimizers on the three studied problems. The first observation is that in each problem, KFAC as well as our methods optimize the training loss function faster than SGD and ADAM both with respect to epoch and time. 
Although our methods may seem much more computationally expensive than KFAC since at each iteration we perform the power SVD or Lanczos bi-diagonalization to estimate the Fisher matrix, they actually have the same order of magnitude in computational cost as KFAC. See appendix \ref{appendix:costs}  for a comparison of the computational costs.
For each of the three problems, we observe that KFAC and KPSVD perform about the same while the DEFLATION, LANCZOS and KFAC-CORRECTED methods have the ability to optimize the objective function much faster both with respect to epoch and time.

Although this is not our object of study, we observe that for each of the three problems, our proposed methods also maintain a good generalization.

%\subsection{Limitations of our methods}
\begin{comment}
We evaluated our methods on others MLP architectures and obtained good results.
However, when considering CNN architectures, we did not observe any gain in performance compared to KFAC. This can be explained by the fact that IAD is not the only assumption made by KFAC for CNNs \cite{GrosseMartens2016} and therefore steering clear from  this hypothesis alone is insufficient to reach a better performance.
\end{comment}

\section{Conclusion and perspectives}
In this work, we proposed a series of novel Kronecker factorizations to the blocks of the Fisher of multi-layer perceptrons using the Kronecker product Singular Value Decomposition technique. Tests realized on the three standard deep auto-encoder problems showed that 3 out of 4 of our proposed methods (DEFLATION, LANCZOS, KFAC-CORRECT) outperform KFAC both in terms of Fisher approximation quality and in terms of optimization speed of the objective function. This ranking, which goes from the most efficient one to the least efficient one, testifies to the fact that higher-rank approximations yield better results than lower-rank ones.

KFAC as well as our methods use a block-diagonal approximation of FIM where each block corresponds to a layer. This results in ignoring the correlations between the layers. Future works will focus on incorporating cross-layer information, as was attempted by \citet{twolevels} with a two-level KFAC preconditioning approach.  

\newpage
%\addcontentsline{toc}{chapter}{Bibliography}
\bibliography{DataSciences}
\bibliographystyle{bib}

\appendix

\section{Proofs} \label{appendix:proof}

\subsection{Proof of Theorem \ref{th:th1}}\label{appendix:th1}
\begin{proof}
We are going to derive the identity 
\begin{equation}\label{eq:targetid}
\|F_{i,i} - R\otimes S\|_F = \|\rearrange(F_{i,i}) - \text{vec}(R)\text{vec}(S)^T\|_F
\end{equation}
for all $R\in\mathbb{R}^{(d_{i-1}+1)\times (d_{i-1}+1)}$ and $S\in\mathbb{R}^{d_i\times d_i}$, from which Theorem \ref{th:th1} will follow. For notational convenience, let
$$
M=F_{i,i}, \qquad d= d_{i-1}+1, \qquad d'= d_i.
$$
We recall that $M$ has the block structure
$$
M =\begin{bmatrix} 
    M_{1,1} & \dots & M_{1,d} \\
     M_{2,1} & \dots & M_{2,d}\\
    \vdots  & & \vdots \\
     M_{d,1}  & \dots & M_{d,d} 
    \end{bmatrix} \in \mathbb{R}^{d'd\times d'd} ,
$$
where each block $M_{\mu,\nu}$, $(\mu,\nu)\in \{1,\ldots, d\}^2$, is of size $d' \times d'$. By definition of the Frobenius norm,
\begin{align}
\|M -R\otimes S\|_{F}^2 & = \sum_{\mu=1}^d\sum_{\nu=1}^d \|M_{\mu,\nu} - R_{\mu,\nu}S\|_F^2\nonumber\\
                       & = \sum_{\mu=1}^d\sum_{\nu=1}^d \|\text{vec}(M_{\mu,\nu}) - R_{\mu,\nu}\text{vec}(S)\|_2^2\nonumber\\
                       & = \sum_{\mu=1}^d\sum_{\nu=1}^d\|\text{vec}(M_{\mu,\nu})^T -R_{\mu,\nu}\text{vec}(S)^T\|_2^2,\label{eq:almostthere}
\end{align}
where $R_{\mu,\nu}$ is the $(\mu,\nu)$-scalar entry of $R$ and $\|\cdot\|_2$ denotes the Euclidean norm. By virtue of                      
$$
\rearrange(M) = \begin{bmatrix}
    \text{vec}(M_{1,1})^T\\
    \vdots\\
    \text{vec}(M_{d,1})^T\\
    \vdots\\
    \text{vec}(M_{1,d})^T\\
    \vdots\\
    \, \text{vec}(M_{d,d})^T\,\\
    \end{bmatrix},
\qquad
\text{vec}(R)\text{vec}(S)^T = \begin{bmatrix}
    R_{1,1}\text{vec}(S)^T\\
    \vdots\\
    R_{d,1}\text{vec}(S)^T\\
    \vdots\\
    R_{1,d}\text{vec}(S)^T\\
    \vdots\\
    R_{d,d}\text{vec}(S)^T\\
    \end{bmatrix},
$$
the last equality of \eqref{eq:almostthere} also reads $\|M - R\otimes S\|^2_F = \|\rearrange(M) - \text{vec}(R)\text{vec}(S)^T\|^2_F$, which proves \eqref{eq:targetid}.
%    $$
%    \begin{aligned}
%    \|A-R\otimes S\|_{F}^2 = & \sum_{i=1}^M\sum_{j=1}^N\|A_{i,j}-[R]_{i,j}S\|_F^2
%                        =   \sum_{i=1}^M\sum_{j=1}^N\|\text{vec}(A_{i,j})-[R]_{i,j}\text{vec}(S)\|_2^2
%                        = \sum_{i=1}^M\sum_{j=1}^N\|\text{vec}(A_{i,j})^T-[R]_{i,j}\text{vec}(S)^T\|_2^2\\
%                        = & \|\rearrange\left(A\right)-\text{vec}\left(R\right)\text{vec}\left(S\right)^T\|_{F}^2
%    \end{aligned}
%    $$
%    
%Therefore: $$\argmin_{R,S} \|A-R\otimes S\|_{F}=\argmin_{R,S} \|\rearrange\left(A\right)-\text{vec}\left(R\right)\text{vec}\left(S\right)^T\|_{F}$$     
\end{proof}

\subsection{Proof of Proposition \ref{prop:1}} \label{appendix:prop1}
\begin{proof}
Using the shorthand notations
$$
\Aat = \Bar{a}_{i-1}\Bar{a}_{i-1}^T, 
\qquad
\Ggt = g_i g_i^T , 
\qquad
d= d_{i-1}+1, \qquad d'= d_i
$$    
we have
$$
F_{i,i}= \mathbb{E}[\Aat \otimes \Ggt] 
     =  \mathbb{E}\left(\begin{bmatrix} 
    \Aat_{1,1}\Ggt & \dots & \Aat_{1,d}\Ggt\\
    \vdots &  & \vdots \\
    \Aat_{d,1}\Ggt & \dots &  \Aat_{d,d}\Ggt
    \end{bmatrix}\right) \in \mathbb{R}^{d' d \times d'd} .
$$
Hence,
$$
\rearrange(F_{i,i})=
     \mathbb{E}\left(
     \begin{bmatrix}
    \text{vec}(\Aat_{1,1}\Ggt)^T\\
    \vdots\\
    \text{vec}(\Aat_{d,1}\Ggt)^T\\
    \vdots\\
     \text{vec}(\Aat_{1,d}\Ggt)^T\\
      \vdots\\
       \text{vec}(\Aat_{d,d}\Ggt)^T\\
    \end{bmatrix}\right) \in  \mathbb{R}^{d^2 \times (d')^2}
$$
For all $v \in \mathbb{R}^{(d')^2}$,
$$
\rearrange(F_{i,i})v =  \mathbb{E}\left(
     \begin{bmatrix}
    \text{vec}(\Aat_{1,1}\Ggt)^T\\
    \vdots\\
    \text{vec}(\Aat_{d,1}\Ggt)^T\\
    \vdots\\
     \text{vec}(\Aat_{1,d}\Ggt)^T\\
      \vdots\\
       \text{vec}(\Aat_{d,d}\Ggt)^T\\
    \end{bmatrix}\right) v 
    =  \mathbb{E}\left(
     \begin{bmatrix}
    \Aat_{1,1}\text{vec}(\Ggt)^Tv\\
    \vdots\\
    \Aat_{d,1}\text{vec}(\Ggt)^Tv\\
    \vdots\\
     \Aat_{1,d}\text{vec}(\Ggt)^Tv\\
      \vdots\\
      \Aat_{d,d} \text{vec}(\Ggt)^Tv
    \end{bmatrix}\right) 
    =  \mathbb{E}[\,(\text{vec}(\Ggt)^Tv)\,\text{vec}(\Aat)] .
$$
The scalar quantity $\text{vec}(\Ggt)^T v$ can be further detailed as
$$
\text{vec}(\Ggt)^T v = (\text{vec}(g_i g_i^T))^T v = (g_i \otimes g_i)^T v = (g_i^T \otimes g_i^T) \,\text{vec}(\text{MAT}(v)), 
$$
owing to the identities $\text{vec}(xy^T) = y\otimes x$ and $(A\otimes B)^T = A^T \otimes B^T$. Invoking now $(A\otimes B)\, \text{vec}(X)= \text{vec}(BXA^T)$, we end up with
$$
\text{vec}(\Ggt)^T v = \text{vec} (g_i^T \,\text{MAT}(v) \, g_i) = \text{vec} (g_i^T V  g_i).
$$
Therefore, $\rearrange(F_{i,i})v = \mathbb{E}[(g_i^T V g_i) \, \text{vec}(\Aat)]$. The proof of
$
\rearrange(F_{i,i})^Tu = \mathbb{E}[(\Bar{a}_{i-1}^TU\Bar{a}_{i-1}) \,\text{vec}(\mathcal{G}_i)]
$ for all  $u \in \mathbb{R}^{d^2}$ goes along the same lines.
\end{proof}

\subsection{Proof of Proposition \ref{prop:2}} \label{appendix:prop2}
\begin{proof}
$\;$\\
$\rhd$ \textit{Symmetry}. By construction and up to a choice of sign, 
$$
\text{vec}(\Bar{A}_{i-1}^{\text{KPSVD}}) = \sqrt{\sigma_1}u_1,
\qquad
\text{vec}(G_i^{\text{KPSVD}}) = \sqrt{\sigma_1}v_1,
$$ 
where $\sigma_1$  is the largest singular value of $\rearrange(F_{i,i})$ associated with left and right singular vectors $(u_1, v_1)$. From the standard SVD properties 
$$
\rearrange(F_{i,i})v_1=\sigma_1 u_1, \qquad \rearrange(F_{i,i})^T u_1=\sigma_1 v_1,
$$
we infer that
$$
\sqrt{\sigma_1} \text{vec}(\Bar{A}_{i-1}^{\text{KPSVD}}) = \rearrange(F_{i,i})v_1  =  \mathbb{E}[\,(g_i^T \text{MAT}(v_1)g_i) \, \text{vec} (\Bar{a}_{i-1} \Bar{a}_{i-1}^T) \, ],
$$
the last equality being a consequence of Proposition \ref{prop:1}. The scalar quantity $g_i^T \text{MAT}(v_1)g_i$ can be moved into the argument of the ``vec'' operator, after which we can permute $\mathbb{E}$ and ``vec'' to obtain
$$
\sqrt{\sigma_1} \text{vec}(\Bar{A}_{i-1}^{\text{KPSVD}}) 
= \mathbb{E}[ \,\text{vec} ( (g_i^T \text{MAT}(v_1)g_i) \, \Bar{a}_{i-1} \Bar{a}_{i-1}^T) \, ]
= \text{vec}(\mathbb{E}[ \, (g_i^T \text{MAT}(v_1)g_i) \, \Bar{a}_{i-1} \Bar{a}_{i-1}^T \, ]).
$$
Hence, upon taking the ``MAT'' operator,
$$
\sqrt{\sigma_1} \Bar{A}_{i-1}^{\text{KPSVD}}
= \mathbb{E}[ \, (g_i^T \text{MAT}(v_1)g_i) \, \Bar{a}_{i-1} \Bar{a}_{i-1}^T \, ].
$$
Since each $(g_i^T \text{MAT}(v_1)g_i) \,\Bar{a}_{i-1} \Bar{a}_{i-1}^T$ is a symmetric matrix, their expectation is also symmetric. The symmetry of $G_i^{\text{KPSVD}}$ is proven in a similar fashion.
    
\noindent $\rhd$ \textit{Positive and semi-definiteness}. The proof of this part is inspired from {\em Theorem 5.8} in \cite{VanLoanPitsianis1993}. Since $\Bar{A}_{i-1}^{\text{KPSVD}}$ and $G_i^{\text{KPSVD}}$ are symmetric, they can be diagonalized as 
\begin{alignat*}{2}
\Bar{A}_{i-1}^{\text{KPSVD}} & = U^T D U, & \qquad D & = \text{diag}(\alpha_1,\alpha_2,\hdots,\alpha_{d_{i-1}+1}),\\
G_i^{\text{KPSVD}} & = V^T E V, & \qquad E & = \text{diag}(\beta_1,\beta_2,\hdots,\beta_{d_i}),
\end{alignat*}
with orthogonal matrices $U$ and $V$. 
%Because $\Bar{a}_{i-1} \Bar{a}_{i-1}^T$ is positive semi-definite, all the $\alpha$'s must have the same sign (as $g_i^T \text{MAT}(v_1)g_i$). Likewise, all the $\beta$'s must have the same sign. 
We are going to show that it is possible to modify the matrices, while preserving minimality of the Frobenius norm, so that the $\alpha$'s and the $\beta$'s all have the same sign. To this end, we first observe that
\[
\Bar{A}_{i-1}^{\text{KPSVD}} \otimes G_i^{\text{KPSVD}} = (U^T D U) \otimes (V^T E V)
= (U\otimes V)^T (D\otimes E) (U\otimes V),
\]
which leads us to introduce
\[
C = (U\otimes V) F_{i,i} \,(U\otimes V)^T.
\]
By unitary invariance of the Frobenius norm, we have
\[
 \|F_{i,i}- \Bar{A}_{i-1}^{\text{KPSVD}} \otimes G_i^{\text{KPSVD}}\|_F^2 =  \| (U\otimes V)^T (C - D \otimes E ) (U\otimes V) \|_F^2
   =  \|C - D \otimes E\|_F^2 .
\]
The last quantity can be expressed as
\[
\|C - D \otimes E\|_F^2 = \sum_{\omega = 1}^{d_i(d_{i-1}+1)} \!\! (C_{\omega,\omega} - (D\otimes E)_\omega)^2 + \sum_{\xi\neq \eta} C_{\xi,\eta}^2
= \sum_{\omega = 1}^{d_i(d_{i-1}+1)} \!\! (C_{\omega,\omega} - \alpha_{\mu(\omega)} \beta_{\tau(\omega)})^2 + \sum_{\xi\neq \eta} C_{\xi,\eta}^2 ,
\]
where $\mu(\omega)\in \{1,\ldots, d_{i-1}+1\}$ and $\tau(\omega)\in \{1,\ldots , d_i\}$ can be uniquely determined\footnote{The solution is given by $\mu(\omega)= 1+ \lfloor (\omega-1)/d_i\rfloor$ and $\tau(\omega) = d_i \lfloor (\omega-1)/d_i\rfloor$, where $\lfloor\cdot\rfloor$ is the integer part, but this does not matter here.} from $\omega \in \{ 1, \ldots, d_i (d_{i-1}+1) \}$ in such a way that $\omega = (\mu(\omega)-1) d_i + \tau(\omega)$.
%Since the $\alpha$'s have the same sign and the $\beta$'s have the same sign, the $\alpha_{\mu(\omega)} \beta_{\tau(\omega)}$'s appearing in the above equality must all have the same sign. On the other hand,

Because $F_{i,i}$ is positive semi-definite, $C$ is also positive semi-definite, which implies that $C_{\omega,\omega}\geq 0$.
%If $\alpha_{\mu(\omega)}\beta_{\tau(\omega)} \geq 0$ for all $\omega$, we have what we claim. Assume that $\alpha_{\mu(\omega)}\beta_{\tau(\omega)} \leq 0$ for all $\omega$. Then, it is readily checked that for all $\omega$,
Thus, for all $\omega$,
\[
(C_{\omega,\omega} - \alpha_{\mu(\omega)}\beta_{\tau(\omega)} )^2 - (C_{\omega,\omega} - |\alpha_{\mu(\omega)}| |\beta_{\tau(\omega)}|)^2 =
2 C_{\omega,\omega} (|\alpha_{\mu(\omega)} \beta_{\tau(\omega)}| - \alpha_{\mu(\omega)} \beta_{\tau(\omega)} ) \geq 0.
\]
This means that if we set, for instance,
\[
R = U^T |D| U , \qquad
S = V^T |E| V ,
\]
with $|D| = \text{diag}(|\alpha_1|,\alpha_2|,\ldots,|\alpha_{d_{i-1}+1}|)$ and $|E| = \text{diag}(|\beta_1|,|\beta_2|,\ldots,|\beta_{d_i}|)$, then
\[
\|F_{i,i} - R\otimes S\|_F^2 \leq \|F_{i,i}- \Bar{A}_{i-1}^{\text{KPSVD}} \otimes G_i^{\text{KPSVD}}\|_F^2 .
\] 
If the inequality were strict, minimality of $(\Bar{A}_{i-1}^{\text{KPSVD}} , G_i^{\text{KPSVD}})$ would be contradicted. Therefore, we must have equality. This entails that $(R,S)$ is another minimizer for which the eigenvalues of $R$, as well as those of $S$, are all non-negative. In such a case, we select this pair $(R,S)$ for the factors $(\Bar{A}_{i-1}^{\text{KPSVD}} , G_i^{\text{KPSVD}})$. 
%This procedure allows us to assume that the $\alpha$'s and the $\beta$'s all have the same sign. In other words, either both matrices are positive semi-definite or both of them are negative semi-definite. Since
%$$
%\Bar{A}_{i-1}^{\text{KPSVD}} \otimes G_i^{\text{KPSVD}} = (-\Bar{A}_{i-1}^{\text{KPSVD}}) \otimes (-G_i^{\text{KPSVD}}),
%$$
%we have the freedom to choose the sign so that both matrices are positive semi-definite.
\end{proof}

\section{Algorithms}\label{appendix:algo}
\subsection{Power SVD algorithm}\label{appendix:svd}

Algorithm to compute the dominant singular value $\sigma_1=\sigma_{\max}$ of a real rectangular matrix and associated right and left singular vectors.

\begin{algorithm}[H]
\SetAlgoLined
\KwInput{$A \in \mathbb{R}^{m\times n}$, $v^{(0)} \in \mathbb{R}^m$, $\epsilon$ (precision), $k_{\max}$ (maximum iteration).}
\KwOutput{ $\sigma_1$, $u_1$ and $v_1$ ($Av_1=\sigma_1u_1$ , $A^Tu_1=\sigma_1v_1$}
 \For{$k=1,2,\hdots,k_{\max}$}{
$w^{(k)} =Av^{(k-1)}$ ;  $u^{(k)}= w^{(k)} /\|w^{(k)}\|_2$\;
  $z^{(k)}=A^Tu^{(k)}$ ; $v^{(k)}= z^{(k)} /\|z^{(k)}\|_2$\;
 $\sigma^{(k)}=\|z^{(k)}\|_2$\;
 $\mathtt{error} = \|Av^{(k)}-\sigma^{(k)}u^{(k)}\|_2$\;
    
  \If{$\mathtt{error}\leq\epsilon$}{
   Break\;
  }

 }
 \caption{SVD Power algorithm}
\end{algorithm}

\subsection{Lanczos bidiagonalization algorithm}\label{appendix:lanczos}
Let $A\in \mathbb{R}^{m\times n}$ be the input matrix. To build a rank-$k$ approximation
\begin{equation}\label{eq:SVDapprox}
A\approx A_K = U_K \Sigma_K V_K^T
\end{equation}
of $A$ \cite{Golub}, where $\Sigma_K \in \mathbb{R}^{K\times K}$ is a diagonal matrix and $U_K\in \mathbb{R}^{n\times K}$, $V_K \in\mathbb{R}^{m\times K}$ are rectangular passage matrices, we first apply Algorithm \ref{algo:lanczos} to obtain the outputs $P \in \mathbb{R}^{n\times K}$, $Q  \in \mathbb{R}^{m\times K}$, $H \in \mathbb{R}^{K\times K}$. The matrix $H$ represents a truncated version of $A$ in another basis associated with $P, Q$. The key observation here is that $H$ is of small size, therefore its SVD
\[
H =  X_K \Sigma_K Y_K^T = \sum_{i=1}^K \sigma_ix_iy_i^T,
\]
with
\[\sigma_1\geq \sigma_2 \geq \hdots \geq \sigma_K,
\qquad X_K\in\mathbb{R}^{K\times K},
\qquad Y_K\in\mathbb{R}^{K\times K},
\]
is not expensive to compute. Going back to the initial basis by the left and right multiplications
\[
A_K  := PHQ^T = PX_K \Sigma_K Y_K^T Q^T ,
\]
we end up with the desired approximation \eqref{eq:SVDapprox} by noticing that
\[
U_K = P X_K, \qquad V_K = Q Y_K . 
\]

\begin{algorithm}[H]\label{algo:lanczos}
\SetAlgoLined
\KwInput{$A \in \mathbb{R}^{m\times n}$, $q^{(0)} \in \mathbb{R}^m, \|q^{(0)}\|=1$, $K$ (dimension of Krylov subspace), $\epsilon$ (precision)}
\KwOutput{Matrices $P \in \mathbb{R}^{n\times K}$, $Q  \in \mathbb{R}^{m\times K}$, $H \in \mathbb{R}^{K\times K}$ and $P \in \mathbb{R}^{K\times K}$}
{\bfseries Start:} \\$w^{(0)}=Aq^{(0)}$\\ $\alpha^{(0)}=\|w^{(0)}\|$\\ $p^{(0)}=w^{(0)} /\alpha^{(0)}$\\ $H[0,0]=\alpha_0$\\ $P[:,0]=p^{(0)}$\\ $Q[:,0]=q^{(0)}$ \\
 \For{$k=0,1,\hdots, K-1$}{
  $z^{(k)}=A^Tp^{(k)}-\alpha^{(k)}q^{(k)}$\\
  $\beta^{(k)}=\|z^{(k)}\|$\\
   \eIf{$\beta^{(k)}\leq\epsilon$}
    {Break}
   {
     $q^{(k+1)}= z^{(k)} /\beta^{(k)}$\;
  $w^{(k+1)}=Aq^{(k+1)}-\beta^{(k)}p^{(k)}$\;
  $\alpha^{(k+1)}=\|w^{(k+1)}\|$\;
  $p^{(k+1)}=w^{(k+1)} /\alpha^{(k+1)}$\;
  $H[k+1,k+1]=\alpha^{(k+1)}$\;
 $H[k,k+1]=\beta^{(k)}$\;
  $P[:,k+1]=p^{(k+1)}$\;
  $Q[:,k+1]=q^{(k+1)}$\;
   }
  
 }
 \caption{Lanczos bidiagonalization algorithm}
\end{algorithm}

\section{Computational costs}\label{appendix:costs}
Here we estimate the computation costs required to compute $\hat{F}$ (estimate of $F$), $\hat{F}^{-1}$ and $\hat{F}^{-1}\nabla h$ of our proposed methods compared to KFAC.
We recall that here $d$ denotes the number of neurons in each layer, $\ell$ denotes the number of network layers and $m$ the mini-batch size. Table \ref{tab:costs} summarizes orders of computational costs required by each method. We did not include forwards and backwards/additional backwards costs as they are the same for all methods. 
$K$ is the dimension of Krylov subspace in Lanczos bi-diagonalization algorithm (see \S\ref{appendix:lanczos}). $k_1$ and $k_2$ represent the number of iterations at which the corresponding algorithm has converged (power SVD or Lanczos bi-diagonalization algorithm). In our experiments, we found that they are of the order of tens. As for $c_1$ and $c_2$ they denote implementation constants.

As we can see in Table \ref{tab:costs}, our proposed methods are of the same order of magnitude as KFAC in terms of computation costs.
\begin{table}[H]
\caption{Range of the computational costs per update.}
\label{tab:costs}
\vskip 0.15in
\begin{center}
\begin{small}
\begin{sc}
\begin{tabular}{lllll}
\toprule
  & $\hat{F}$ & $\hat{F}^{-1}$ & $\hat{F}^{-1}\nabla h$\\
\midrule
KFAC    & $2\ell m d^2$ & $2\ell d^3$& $2\ell d^3$ \\
KPSVD & $4k_1\ell md^2 $& $2\ell d^3$ & $2\ell d^3$\\
Deflation    & $4k_1\ell md^2+4k_1\ell md^2$ & $c_1\ell d^3$& $c_2\ell d^3$ \\
Lanczos   & $4k_2\ell md^2$ $+\ell K^3+2\ell Kd^2$ & $c_1\ell d^3$& $c_2\ell d^3$ \\ 
KFAC-corrected    & $2\ell m d^2+4k_1\ell md^2$ & $c_1\ell d^3$& $c_2\ell d^3$ \\
\bottomrule
\end{tabular}
\end{sc}
\end{small}
\end{center}
\vskip -0.1in
\end{table}

\subsection*{Explanation of the entries of Table \ref{tab:costs}}
\begin{itemize}
    \item \textbf{KFAC}: To compute $\hat{F}$, we need to compute $2\ell$ terms $\Bar{A}_{i-1}=\mathbb{E}[\Bar{a}_{i-1}\Bar{a}_{i-1}^T]$ and $G_i = \mathbb{E}[g_ig_i^T]$ of computational costs $O(md^2)$ each. For $\hat{F}^{-1}$, the inverses of the $\ell$ pairs $\Bar{A}_{i-1}$ and $G_i$ are required. The computational cost of each $\Bar{A}_{i-1}^{-1}$ or $G_i^{-1}$ is $O(d^3)$. As for $\hat{F}^{-1}\nabla h$, we need to perform $\ell$ matrix-matrix multiplications $G_i^{-1}\nabla_WhA_{i-1}^{-1}$ (see equation (\ref{eq:precond})).
    
    \item \textbf{KPSVD}: The computation of $\hat{F}$ requires to apply the power SVD algorithm. If $k_1$ is the iteration number of convergence, then for each layer $i$, we need to perform $k_1$ matrix-vector multiplications $\rearrange(F_{i,i})v=\mathbb{E}[(g_i^TVg_i)\text{vec}(\Bar{a}_{i-1}\Bar{a}_{i-1}^T)] $ and $\rearrange(F_{i,i})^Tu=\mathbb{E}[(\Bar{a}_{i-1}^TU\Bar{a}_{i-1})\text{vec}(g_ig_i^T)]$. The computational cost of $\rearrange(F_{i,i})v$ or $\rearrange(F_{i,i})^Tu$ is $O(md^2)$. The computational costs required for $\hat{F}^{-1}$ and $\hat{F}^{-1}\nabla h$ are the same as in KFAC.

    \item \textbf{KFAC-CORRECTED}: The computation of $\hat{F}$ is a combination of the computation of $\hat{F}$ in KFAC and in KPSVD so the complexity is the sum of the complexity in KFAC and KPSVD. As for $\hat{F}^{-1}$ and $\hat{F}^{-1}\nabla h$ the technique described in subsection \ref{sec:inversion} is used and the complexities are $O(c_1\ell d^3)$ for  $\hat{F}^{-1}$ (SVD and matrix-matrix multiplications) and $O(c_2\ell d^3)$ for $\hat{F}^{-1}\nabla h$ (matrix-matrix multiplications).
    
    \item \textbf{DEFLATION}: To compute $\hat{F}$ for a single layer, we have applied twice the power SVD algorithm and each application has the same cost as in KPSVD.  So the total computational cost of computing $\hat{F}$ in DEFLATION is twice the total computational cost of computing $\hat{F}$ in KPSVD. The computational costs required for $\hat{F}^{-1}$ and $\hat{F}^{-1}\nabla h$ are the same as in KFAC-CORRECTED.
    
    \item \textbf{LANCZOS}: To compute $\hat{F}$, the Lanczos bi-diagonalization algorithm is applied for each layer. Like in KPSVD, if $k_2$ is the iteration number of convergence then $k_2$ $\rearrange(F_{i,i})v$ and $\rearrange(F_{i,i})^Tu$ (in $O(md^2)$ each) were necessary for each layer. At the end Lancozs of the bi-diagonalization algorithm, we need to perform for each layer, the SVD of matrix $H \in \mathbb{R}^{K\times K}$ (in $O(K^3)$) and matrix-matrix operations $PX_k$ (in $O(Kd^2)$) and $QY_k$ (in $O(Kd^2)$).The computational costs required for $\hat{F}^{-1}$ and $\hat{F}^{-1}\nabla h$ are the same as in DEFLATION or KFAC-CORRECTED.
    
\end{itemize}

\section{Network architectures and Datasets} \label{appendix:data}
We describe here the datasets and network architectures \cite{hinton2006} used in our tests.

\begin{itemize}
\item \textbf{Auto-encoder problem 1}
    
        \begin{itemize}
            \item Network architecture: $784-1000-500-250-30-250-500-1000-784$
            \item Activations functions: $\text{ReLU}-\text{ReLU}-\text{ReLU}-\text{ReLU}-\text{ReLU}-\text{ReLU}-\text{ReLU}-\text{Sigmoid}$
            \item Data : MNIST (images of shape $28\times 28$ of handwritten digits. $50000$ training images and $10000$ validation images). 
            \item Loss function: binary cross entropy
        \end{itemize}
        
     \item \textbf{Auto-encoder problem 2}   
        \begin{itemize}
           \item Network architecture: $625-2000-1000-500-30-500-1000-2000-625$
           \item Activation functions: $\text{ReLU}-\text{ReLU}-\text{ReLU}-\text{ReLU}-\text{ReLU}-\text{ReLU}-\text{ReLU}-\text{Linear}$
            \item Data : FACES (images of shape $25\times 25$ people. $82800$ training images and $20700$ validation images). 
            \item Loss function: mean square error.
        \end{itemize}

\item \textbf{Auto-encoder problem 3}
    \begin{itemize}
        \item Network architecture: $784-400-200-100-50-25-6-25-50-100-200-400-784$
        \item Activations functions: $\text{ReLU}-\text{ReLU}-\text{ReLU}-\text{ReLU}-\text{ReLU}-\text{ReLU}-\text{ReLU}-\text{ReLU}-\text{ReLU}-\text{ReLU}-\text{ReLU}-\text{Sigmoid}$
            \item Data : CURVES (images of shape $28\times 28$ of simulated handdrawn
    curves. $16000$ training images and $4000$ validation images). 
            \item Loss function: binary cross entropy.
    \end{itemize}

\end{itemize}

\section{Gradient clipping} \label{appendix:clipping}
We applied the KL-clipping technique \cite{BaGrosseMartens2017}: after preconditioning the gradients, we scaled them by a factor $\nu$ given by
$$
\nu = \min\bigg(1,\sqrt{\frac{c}{\sum_{i=1}^{\ell}|\mathcal{G}_i^T\nabla h(W_i)|}}\bigg) ,
$$
where $\mathcal{G}_i$ denotes the preconditioned gradient and $c$ is a constant that represents the maximum clipping parameter.

\end{document}